\documentclass[11pt]{article}

\usepackage[letterpaper, left=1in, right=1in, top=1in,bottom=1in]{geometry}

\usepackage{parskip}

\usepackage{booktabs} 

\usepackage[utf8]{inputenc}

\usepackage{geometry}

\usepackage{graphpap,amscd,mathrsfs,graphicx,lscape,enumitem,dsfont,bm,url,subfigure}
\usepackage{epsfig,amstext,xspace}
\usepackage{algorithm,comment}
\usepackage{algorithmic}

\usepackage{auxiliary}
\usepackage{thmtools,thm-restate}

\usepackage{algorithm,algorithmic}
\usepackage[utf8]{inputenc} 
\usepackage[T1]{fontenc}    
\usepackage{hyperref}       
\usepackage{url}            
\usepackage{booktabs}       
\usepackage{amsfonts}       
\usepackage{nicefrac}       
\usepackage{microtype}      

\usepackage{color}              
\usepackage[suppress]{color-edits}
\addauthor{tl}{cyan}
\addauthor{sg}{red}
\addauthor{ab}{brown}
\addauthor{ns}{green}

\begin{document}
\title{On preserving non-discrimination when combining expert advice}

 \author{
  Avrim Blum \thanks{Toyota Technological Institute at Chicago, \texttt{avrim@ttic.edu}. Work supported under NSF grant CCF-1800317.}
  \and Suriya Gunasekar \thanks{Toyota Technological Institute at Chicago, \texttt{suriya@ttic.edu}} \and 
 Thodoris Lykouris\thanks{Cornell University, \texttt{teddlyk@cs.cornell.edu}. Work supported under NSF grant CCF-1563714 and a Google Ph.D. Fellowship} 
 \and Nathan Srebro \thanks{Toyota Technological Institute at Chicago, \texttt{nati@ttic.edu}}
 }
\date{}
\maketitle

\begin{abstract}
We study the interplay between sequential decision making and avoiding discrimination against protected groups, when examples arrive online and do not follow distributional assumptions. We consider the most basic extension of classical online learning: \emph{Given a class of predictors that are individually non-discriminatory with respect to a particular metric, how can we combine them to perform as well as the best predictor, while preserving non-discrimination?} Surprisingly we show that this task is unachievable for the prevalent notion of \emph{equalized odds} that requires equal false negative rates and equal false positive rates across groups. On the positive side, for another notion of non-discrimination, \emph{equalized error rates}, we show that running separate instances of the classical multiplicative weights algorithm for each group achieves this guarantee. Interestingly, even for this notion, we show that algorithms with stronger performance guarantees than  multiplicative weights cannot preserve non-discrimination. 
\end{abstract}

\addtocounter{page}{-1}
\thispagestyle{empty}
\newpage

\section{Introduction}
\label{sec:intro}

The emergence of machine learning in the last decade has given rise to an important debate regarding the ethical and societal responsibility of its offspring. Machine learning has provided a universal toolbox enhancing the decision making in many disciplines from advertising and recommender systems to education and criminal justice. Unfortunately, both the data and their processing can be biased against specific population groups (even inadvertently) in every single step of the process \cite{barocas2014datas}. This has generated societal and policy interest in understanding the sources of this discrimination and interdisciplinary research has attempted to mitigate its shortcomings.

Discrimination is commonly an issue in applications where decisions need to be made sequentially. The most prominent such application is online advertising where platforms need to sequentially select which ad to display in response to particular query searches. This process can introduce discrimination against protected groups in many ways such as filtering particular alternatives \cite{datta2015automated,angwin2016facebook} and reinforcing existing stereotypes through search results \cite{Sweeney2013,kay2015unequal}. Another canonical example of sequential decision making is medical trials where underexploration on female groups often leads to significantly worse treatments for them \cite{liu2016women}. Similar issues occur in image classification as stressed by ``gender shades'' \cite{buolamwini2018gender}. The reverse (overexploration in minority populations) can also cause concerns especially if conducted in a non-transparent fashion \cite{bird2016exploring}.

In these sequential settings, the assumption that data are i.i.d. is often violated. Online advertising, recommender systems, medical trials, image classification, loan decisions, criminal recidivism all require decisions to be made sequentially. The corresponding labels are not identical across time and can be affected by the economy, recent events, etc. Similarly labels are also not independent across rounds -- if a bank offers a loan then this decision can affect whether the loanee or their environment will be able to repay future loans thereby affecting future labels as discussed by Liu et al. \cite{LiuDeRoSiHa18}. As a result, it is important to understand the effect of this adaptivity on non-discrimination.

The classical way to model settings that are not i.i.d. is via adversarial online learning \cite{Littlestone:1994:WMA:184036.184040,Freund1997}, which poses the question: \emph{Given a class $\mathcal{F}$ of predictors, how can we make online predictions that perform as well as the best predictor from $\mathcal{F}$ in hindsight?} The most basic online learning question (answered via the celebrated ``multiplicative weights'' algorithm) concerns competing with a finite set of predictors. The class $\mathcal{F}$ is typically referred to as “experts” and can be thought as ``features'' of the example where we want to make online predictions that compete with the best 1-sparse predictor.

In this work, we wish to understand the interplay between adaptivity and non-discrimination and therefore consider the most basic extension of the classical online learning question:
\begin{quote}
\begin{center}
\emph{Given a class of \textbf{individually non-discriminatory predictors}, how can we combine them to perform as well as the best predictor, while preserving non-discrimination?}
\end{center}
\end{quote}
The assumption that predictors are individually non-discriminatory is a strong assumption on the predictors and makes the task trivial in the batch setting where the algorithm is given labeled examples and wishes to perform well on unseen examples drawn from the same distribution. This happens because the algorithm can learn the best predictor from the labeled examples and then follow it (since this predictor is individually non-discriminatory, the algorithm does not exhibit discrimination). This enables us to understand the potential overhead that adaptivity is causing and significantly strengthens any impossibility result. Moreover, we can assume that predictors have been individually vetted to satisfy the non-discrimination desiderata -- we therefore wish to understand how to efficiently compose these non-discriminatory predictors while preserving non-discrimination.

\subsection{Our contribution}

\paragraph{Our impossibility results for equalized odds.} Surprisingly, we show that for a prevalent notion of non-discrimination, \emph{equalized odds}, it is impossible to preserve non-discrimination while also competing comparably to the best predictor in hindsight (no-regret property). Equalized odds, suggested by Hardt et al. \cite{hardt2016equality} in the batch setting, restricts the set of allowed predictors requiring that, when examples come from different groups, the prediction is independent to the group conditioned on the label. 
In binary classification, this means that the false negative rate (fraction of positive examples predicted negative) is equal across groups and the same holds for the false positive rate (defined analogously). This notion was popularized by a recent debate on potential bias of machine learning risk tools for criminal recividism
\cite{angwin2016machine,chouldechova2017fair,Kleinberg2017InherentTI, feller2016computer}. 

Our impossibility results demonstrate that the order in which examples arrive significantly complicates the task of achieving desired efficiency while preserving non-discrimination with respect to equalized odds. In particular, we show that any algorithm agnostic to the group identity either cannot achieve performance comparable to the best predictor or exhibits discrimination in some instances (Theorem~\ref{thm:eq_odds_negative_group_unaware}).  This occurs in phenomenally simple settings with only two individually non-discriminatory predictors, two groups, and perfectly balanced instances: groups are of equal size and each receives equal number of positive and negative labels. The only imbalance exists in the order in which the labels arrive which is also relatively well behaved -- labels are generated from two i.i.d. distributions, one in the first half of the instance and one in the second half. Although in many settings we cannot actively use the group identity of the examples due to legal reasons (e.g., in hiring), one may wonder whether these impossibility results disappear if we can actively use the group information to compensate for past mistakes. We show that this is also not the case (Theorem~\ref{thm:eq_odds_negative_imbalance}). Although our groups are not perfectly balanced, the construction is again very simple and consists only of two groups and two predictors: one always predicting positive and one always predicting negative. The simplicity of the settings, combined with the very strong assumption on the predictors being individually non-discriminatory speaks to the trade-off between adaptivity and non-discrimination with respect to equalized odds.

\paragraph{Our results for equalized error rates.} The strong impossibility results with respect to equalized odds invite the natural question of whether there exists some alternative fairness notion that, given access to non-discriminatory predictors, achieves efficiency while preserving non-discrimination. We answer the above positively by suggesting the notion of \emph{equalized error rates}, which requires that the average expected loss (regardless whether it stems from false positives or false negatives) encountered by each group should be the same. This notion makes sense in settings where performance and fairness are measured with respect to the same objective. Consider a medical application where people from different subpopulations wish to receive appropriate treatment and any error in treatment costs equally both towards performance and towards fairness.\footnote{In contrast, in equalized odds, a misprediction is only costly to the false negative metric if the true label is positive.}  It is morally objectionable to discriminate against one group, e.g. based on race, using it as experimentation to enhance the quality of service of the other, and it is reasonable to require that all subpopulations receive same quality of service.

For this notion, we show that, if all predictors are individually non-discriminatory with respect to equalized error rates, running separate multiplicative weights algorithms, one for each subpopulation, preserves this non-discrimination without decay in the efficiency (Theorem \ref{thm:eq_error_rates_positive}). The key property we use is that the multiplicative weights algorithm guarantees to perform not only no worse than the best predictor in hindsight but also no better; this property holds for a broader class of algorithms \tledit{\cite{gofer2016lower}}
. Our result applies to general loss functions beyond binary predictions and only requires predictors to satisfy the weakened assumption of being approximately non-discriminatory. 

Finally, we examine whether the decisions of running separate algorithms and running this particular not so efficient algorithm were important for the result. We first give evidence that running separate algorithms is essential for the result; if we run a single instance of ``multiplicative weights'' or ``follow the perturbed leader'', we cannot guarantee non-discrimination with respect to equalized error rates (Theorem \ref{thm:single_algo_error_rates}). We then suggest that the property of not performing better than the best predictor is also crucial; in particular, better algorithms that satisfy the stronger guarantee of low shifting regret \cite{DBLP:journals/jmlr/HerbsterW01,DBLP:conf/colt/BlumM05,DBLP:conf/colt/LuoS15} are also not able to guarantee this non-discrimination (Theorem \ref{thm:shifting_algo_error_rates}). These algorithms are considered superior to classical no-regret algorithms as they can better adapt to changes in the environment, which has nice implications in game-theoretic settings \cite{LykourisSyTa2016}. Our latter impossibility result is a first application where having these strong guarantees against changing benchmarks is not necessarily desired and therefore is of independent learning-theoretic interest.

\subsection{Related work}
\label{sec:related_work}

There is a large line of work on fairness and non-discrimination in machine learning (see \cite{Pedr08eshiRT20,Calders09buildingclassifiers,DworkHPRZ2012awareness,zemel2013learning, JosephKMR2016fairness,hardt2016equality,chouldechova2017fair,Kleinberg2017InherentTI,kearns2017gerrymandering} for a non-exhaustive list). We elaborate on works that either study group notions of fairness or fairness in online learning.

The last decade has seen a lot of work on group notions of fairness, mostly in classification setting. Examples include notions that compare the percentage of members predicted positive such as demographic parity \cite{Calders09buildingclassifiers}, disparate impact \cite{FeldmanFMSV2015disparate}, equalized odds \cite{hardt2016equality} and calibration across groups \cite{chouldechova2017fair,Kleinberg2017InherentTI}. There is no consensus on a universal fairness notion; rather the specific notion considered is largely task-specific. In fact, previous works identified that these notions are often not compatible to each other \cite{chouldechova2017fair,Kleinberg2017InherentTI}, posed concerns that they may introduce unintentional discrimination \cite{corbett2018measure}, and suggested the need to go beyond such observational criteria via causal reasoning \cite{kilbertus2017avoiding, KusnerLoRuSi2017}. Prior to our work, group fairness notions have been studied primarily in the batch learning setting with the goal of optimizing a loss function subject to a fairness constraint either in a post-hoc correction framework as proposed by Hardt et al. \cite{hardt2016equality} or more directly during training from batch data \cite{zemel2013learning,goh2016satisfying,woodworth2017learning,zafar2015learning,balcan_envyfree} which requires care due to the predictors being discriminatory with respect to the particular metric of interest. The setting we focus on in this paper does not have the challenges of the above since all predictors are non-discriminatory; however, we obtain surprising impossibility results due to the ordering in which labels arrive. 

Recently fairness in online learning has also started receiving attention. One line of work  focuses on imposing a particular fairness guarantee \emph{at all times} for bandits and contextual bandits, either for individual fairness \cite{JosephKMR2016fairness, Kannan:2017:FIM:3033274.3085154} or for group fairness \cite{DBLP:journals/corr/CelisV17}. Another line of work points to counterintuitive externalities of using contextual bandit algorithms agnostic to the group identity and suggest that heterogeneity in data can replace the need for exploration \cite{RaghavanSlWoWu18,KannanMoRoWaWu18}. Moreover, following a seminal paper by Dwork et al. \cite{DworkHPRZ2012awareness}, a line of work aims to treat similar people similarly in online settings \cite{DBLP:journals/corr/LiuRDMP17,gillen2018online}. Our work distinguishes itself from these directions mainly in the objective, since we require the non-discrimination to happen \emph{in the long-term} instead of at any time; this extends the classical batch definitions of non-discrimination in the online setting. In particular, we only focus on situations where there are enough samples from each population of interest and we do not penalize the algorithm for a few wrong decisions, leading it to be overly pessimistic. Another difference is that previous work focuses either on individual notions of fairness or on i.i.d. inputs, while our work is about non-i.i.d. inputs in group notions of fairness.

\section{Model}
\label{sec:model}
\paragraph{Online learning protocol with group context.} We consider the classical online learning setting of prediction with expert advice, where a learner needs to make sequential decisions  for $T$ rounds by combining the predictions of a finite set $\mathcal{F}$ of $d$ hypotheses (also referred to as \emph{experts}). We denote the outcome space by $\mathcal{Y}$; in binary classification, this corresponds to $\mathcal{Y}=\crl{+,-}$. Additionally, we introduce a set of disjoint groups by $\mathcal{G}$ which identifies subsets of the population based on a protected attribute (such as gender, ethnicity, or income). 

The \emph{online learning protocol with group context} proceeds in $T$ rounds. Each round $t$ is associated with a group context $g(t)\in\mathcal{G}$ and an outcome $y(t)\in\mathcal{Y}$. We denote the resulting $T$-length time-group-outcome sequence tuple by $\sigma=\crl*{\prn*{t, g(t), y(t)}\in \mathbb{N}\times \mathcal{G}\times\mathcal{Y}}_{t=1}^T$. This is a random variable that can depend on the randomness in the generation of the groups and the outcomes. We use the shorthand $\sigma^{1:\tau}=\crl*{\prn*{t, g(t), y(t)}\in \mathbb{N}\times \mathcal{G}\times\mathcal{Y}}_{t=1}^\tau$ to denote the subsequence until round $\tau$. The exact protocol for generating these sequences is described below.  At round $t=1,2,\dots,T$:
\begin{enumerate}
    \item An example with group context $g(t)\in\mathcal{G}$ either arrives stochastically or is adversarially selected. 
    \item The learning algorithm or \emph{learner} $\mathcal{L}$  commits to a probability distribution $p^t\in \Delta(d)$ across experts where $p_f^t$ denotes the probability that she follows the advice of expert $f\in\mathcal{F}$ at round $t$. This distribution $p^t$ can be a function of the sequence $\sigma^{1:t-1}$. We call the learner \emph{group-unaware} if she ignores the group context $g(\tau)$ for all $\tau\leq t$ when selecting $p^t$.
    \item An adversary $\mathcal{A}$ then selects an outcome $y(t)\in\mathcal{Y}$. The adversary is called \emph{adaptive} if the groups/outcomes at round $t=\tau+1$ are a function of the realization of $\sigma^{1:\tau}$; otherwise she is called \emph{oblivious}. The adversary always has access to the learning algorithm,  but an adaptive adversary   additionally has access to the realized $\sigma^{1:t-1}$ and hence also knows $p^t$.
    
    Simultaneously, each expert $f\in\mathcal{F}$ makes a prediction $\hat{y}_f^t\in\hat{\mathcal{Y}}$, where $\hat{\mathcal{Y}}$ is a generic prediction space; for example, in binary classification, the prediction space could simply be the  positive or negative labels: $\hat{\mathcal{Y}}=\crl{+,-}$, or  the probabilistic score: $\hat{\mathcal{Y}}=[0,1]$ with $\hat{y}_f^t$ interpreted as the probability the expert $f\in\mathcal{F}$ assigns to the positive label in round $t$, or even an uncalibrated score  like the output of a support vector machine: $\hat{\mathcal{Y}}=\mathbb{R}$.
    
    Let $\ell:\hat{\mathcal{Y}}\times\mathcal{Y}\rightarrow \brk{0,1}$ be the loss function between predictions and outcomes. This leads to a corresponding loss vector $\ell^t\in[0,1]^d$ where $\ell_f^t=\ell\prn*{\hat{y}_f^t,y(t)}$ denotes the  loss the learner incurs if she follows expert $f\in\mathcal{F}$. 
    \item The learner then
    observes the entire loss vector $\ell^t$ (full information feedback) and
    incurs expected loss $\sum_{f\in\mathcal{F}} p_f^t\ell_f^t$. For classification, this feedback is obtained by observing~$y(t)$.
\end{enumerate}

In this paper, we consider a setting where all the experts $f\in\mathcal{F}$ are fair in isolation (formalized below). Regarding the group contexts, our main impossibility results (Theorems~\ref{thm:eq_odds_negative_group_unaware} and~\ref{thm:eq_odds_negative_imbalance}) assume that the group contexts $g(t)$ arrive stochastically from a fixed distribution, while our positive result (Theorem~\ref{thm:eq_error_rates_positive}) holds even when they are adversarially selected.

For simplicity of notation, we assume throughout the presentation that the learner's algorithm is producing the distribution $p^t$ of round $t=\tau+1$ deterministically based on $\sigma^{1:\tau}$ and therefore all our expectations are taken only over $\sigma$ which is the case in most algorithms. Our results extend when the algorithm uses extra randomness to select the distribution.

\paragraph{Group fairness in online learning.} We now define non-discrimination (or fairness) with respect to a particular evaluation metric $\mathcal{M}$, e.g. in classification, the false negative rate metric (FNR) is the fraction of examples with positive outcome predicted negative incorrectly. For any realization of the time-group-outcome sequence $\sigma$ and any group $g\in\mathcal{G}$, metric  $\mathcal{M}$ induces a subset of the population $\mathcal{S}^{\sigma}_g(\mathcal{M})$ that is relevant to it. For example, in classification, $\mathcal{S}_g^{\sigma}(FNR)=\crl{t: g(t)=g, y(t) = +}$ is the set of positive examples of group $g$. The performance of expert $f\in\mathcal{F}$ on the subpopulation  $S_g^\sigma(\mathcal{M})$ is denoted by $\mathcal{M}_{f}^{\mathcal{\sigma}}(g)=\frac{1}{|\mathcal{S}_g^{\sigma}(\mathcal{M})|}\sum_{t\in\mathcal{S}_g^\sigma \prn{\mathcal{M}}} \ell_f^t$. 
\begin{definition}
\vspace{0.3em}
An expert $f\in\mathcal{F}$ is called \textbf{fair in isolation with respect to metric $\mathcal{M}$} if, for every sequence $\sigma$, her performance with respect to $\mathcal{M}$ is the same across groups, i.e. 
$\mathcal{M}^{\sigma}_f\prn{g}= \mathcal{M}^{\sigma}_f(g')$ for all $g,g'\in\mathcal{G}$.
\end{definition}
Similarly, the learner's performance on this subpopulation is $\mathcal{M}^{\sigma}_{\mathcal{L}}(g)=\frac{1}{|\mathcal{S}_g^{\sigma}(\mathcal{M})|}\sum_{t\in\mathcal{S}_g^\sigma \prn{\mathcal{M}}} \sum_{f\in\mathcal{F}}p_f^t\ell_f^t$. To formalize our non-discrimination desiderata, we require the algorithm to have similar expected performance across groups, when given access to fair in isolation predictors. We make the following assumptions to avoid trivial impossibility results due to low-probability events or underrepresented populations. First, we take expectation over sequences generated by the adversary $\mathcal{A}$ (that has access to the learning algorithm $\mathcal{L}$).
Second, we require the relevant subpopulations to be, in expectation, \emph{large enough}. Our positive results do not depend on either of these assumptions. More formally:
\begin{definition}
\vspace{0.3em}
Consider a set of experts $\mathcal{F}$ such that each expert is fair in isolation with respect to metric $\mathcal{M}$. Learner $\mathcal{L}$ is called \textbf{$\alpha$-fair in composition with respect to metric $\mathcal{M}$} if, for all adversaries that produce $\En_{\sigma}\brk{\min\prn{|S_g^{\sigma}(\mathcal{M})|,|S_{g'}^{\sigma}(\mathcal{M})|}}=\Omega(T)$ for all $g,g'$, it holds that: $$\abs*{\mathbb{E}_{\sigma} \brk*{\mathcal{M}_{\mathcal{L}}^{\sigma}(g)}-\mathbb{E}_{\sigma}\brk*{\mathcal{M}_{\mathcal{L}}^{\sigma}(g')}}\leq \alpha.$$
\end{definition}
We note that, in many settings, we wish to have non-discrimination with respect to multiple metrics simultaneously. For instance, equalized odds requires fairness in composition both with respect to false negative rate and with respect to false positive rate (defined analogously). Since we provide an impossibility result for equalized odds, focusing on only one metric makes the result even stronger.

\paragraph{Regret notions.} The typical way to evaluate the performance of an algorithm in online learning is via the notion of \emph{regret}. Regret is comparing the performance of the algorithm to the performance of the best expert in hindsight on the realized sequence $\sigma$ as defined below:
$$
\regret_T=\sum_{t=1}^T \sum_{f\in\mathcal{F}} p_f^t\ell_f^t - \min_{f^{\star}\in\mathcal{F}}\sum_{t=1}^T \ell_{f^{\star}}^t.
$$
In the above definition, regret is a random variable depending on the sequence $\sigma$; therefore depending on the randomness in groups/outcomes.

An algorithm satisfies the no-regret property (or Hannan consistency) in our setting if for any losses realizable by the above protocol, the regret is sublinear in the time horizon $T$, i.e. $\regret_T=o(T)$.  This property ensures that, as time goes by, the average regret vanishes. Many online learning algorithms, such as multiplicative weights updates satisfy this property with $\regret_T=O(\sqrt{T\log(d)})$.

We focus on the  notion of \emph{approximate regret}, which is a relaxation of regret that gives a small multiplicative slack to the algorithm. More formally, $\epsilon$-approximate regret with respect to expert $f^{\star}\in\mathcal{F}$ is defined as:
$$
\apx_{\epsilon,T}(f^{\star})=\sum_{t=1}^T\sum_{f\in\mathcal{F}}p_f^t\ell_f^t-(1+\epsilon)\sum_{t=1}^T \ell_{f^{\star}}^t.
$$
We note that typical algorithms guarantee $\apx_{\epsilon,T}(f^{\star})=O(\nicefrac{\ln(d)}{\epsilon})$ simultaneously for all experts $f^{\star}\in\mathcal{F}$. When the time-horizon is known in advance, by setting $\epsilon= \sqrt{\nicefrac{\ln(d)}{T}}$, such a bound implies the aforementioned regret guarantee. In the case when the time horizon is not known, one can also obtain a similar guarantee by adjusting the learning rate of the algorithm appropriately. 

Our goal is to develop online learning algorithms that combine fair in isolation experts in order to achieve both vanishing average expected $\epsilon$-approximate regret, i.e. for any fixed $\epsilon>0$ and $f^{\star}\in\mathcal{F}$, $\mathbb{E}_{\sigma}\brk{\apx_{\eps,T}(f^{\star})}=o(T)$, and also non-discrimination with respect to fairness metrics of interest.

\section{Impossibility results for equalized odds}
\label{sec:equalized_odds}

In this section, we study a popular group fairness notion, equalized odds, in the context of online learning. A natural extension of equalized odds for online settings would require that the false negative rate, i.e. percentage of positive examples predicted incorrectly, is the same across all groups and the same also holds for the false positive rate. We assume that our experts are fair in isolation with respect to both false negative as well as false positive rate. A weaker notion of equalized odds is \emph{equality of opportunity} where the non-discrimination condition is required to be satisfied only for the false negative rate. We first study whether it is possible to achieve the vanishing regret property while guaranteeing $\alpha$-fairness in composition with respect to false negative rate for arbitrarily small $\alpha$. When the input is i.i.d., this
is trivial as we can learn the best expert in $O(\log d)$ rounds and then follow its advice; since the expert is fair in isolation, this will guarantee vanishing non-discrimination.

In contrast, we show that, in a non-i.i.d. online setting, this goal is unachievable. We demonstrate this in phenomenally benign settings where there are just two groups $\mathcal{G}=\crl{A,B}$ that come from a fixed distribution and just two experts that are fair in isolation (with respect to false negative rate) even per round -- not only ex post. Our first construction (Theorem~\ref{thm:eq_odds_negative_group_unaware}) shows that any no-regret learning algorithm that is group-unaware cannot guarantee fairness in composition, even in instances that are perfectly balanced (each pair of label and group gets $\nicefrac{1}{4}$ of the examples) -- the only adversarial component is the order in which these examples arrive. This is surprising because such a task is straightforward in the stochastic setting as all hypotheses are non-discriminatory. We then study whether actively using the group identity can correct the aforementioned similarly to how it enables correction against discriminatory predictors \cite{hardt2016equality}. The answer is negative even in this scenario (Theorem~\ref{thm:eq_odds_negative_imbalance}): if the population is sufficiently not balanced, any no-regret learning algorithm will be unfair in composition with respect to false negative rate even if they are not group-unaware.

\subsection{Group-unaware algorithms}
We first present the theorem about group-unaware algorithms.

\begin{theorem}\label{thm:eq_odds_negative_group_unaware}
\vspace{0.3em}
For all $\alpha<\nicefrac{3}{8}$, there exists $\epsilon>0$ such that any group-unaware algorithm that satisfies $\En_{\sigma}\brk*{\apx_{\eps,T}(f)}=o(T)$ for all $f\in\mathcal{F}$ is $\alpha$-unfair in composition with respect to false negative rate even for perfectly balanced sequences. In particular, for any group-unaware algorithm that ensures vanishing approximate regret\footnote{This  requirement is weaker than vanishing regret so the impossibility result applies to vanishing regret algorithms.}, there exists an oblivious adversary for assigning labels such that:
\begin{itemize}
    \item In expectation, half of the population corresponds to each group.
    \item For each group, in expectation half of its labels are positive and the other half are negative.
    \item The false negative rates of the two groups differ by
    $\alpha$.
\end{itemize}
\end{theorem}
\begin{proof}
Consider an instance that consists of two groups $\mathcal{G}=\crl{A,B}$, two experts $\mathcal{F}=\crl{h_n,h_u}$, and two phases: Phase I and Phase II. Group $A$ is the group we end up discriminating against while group $B$ is boosted by the discrimination with respect to false negative rate. At each round $t$ the groups arrive stochastically with probability $\nicefrac{1}{2}$ each, independent of $\sigma^{1:t-1}$.

The experts output a score value in $\hat{\mathcal{Y}}=[0,1]$, where score $\hat{y}_f^t\in\hat{\mathcal{Y}}$ can be interpreted as the probability that expert $f$ assigns to label being positive in round $t$, i.e. $y(t)=+$. The loss function is the expected probability of error given by $\ell(\hat{y},y)=\hat{y}\cdot\mathbf{1}\{y=-\}+(1-\hat{y})\cdot\mathbf{1}\{y=+\}$. 
 The two experts are very simple:  $h_{n}$ always predicts negative, i.e. $\hat{y}_{h_n}^t=0$ for all $t$, and $h_{u}$ is an unbiased expert who, irrespective
 of the group or the label, makes an inaccurate prediction with probability  $\beta=\nicefrac{1}{4}+\sqrt{\eps}$, i.e. $\hat{y}_{h_u}^t=\beta\cdot \mathbf{1}\{y(t)=-\}+(1-\beta)\cdot \mathbf{1}\{y(t)=+\}$ for all $t$. Both experts are fair in isolation with respect to both false negative and false positive rates:  FNR is $100\%$ for $h_n$ and $\beta$ for $h_u$ regardless the group, and FPR is $0\%$ for $h_n$ and $\beta$ for $h_u$, again independent of the group. The instance proceeds in two phases:

\begin{enumerate}
    \item Phase I lasts for $\nicefrac{T}{2}$ rounds. The adversary assigns negative labels on examples with group context $B$ and assigns a label uniformly at random to examples from group $A$. 
    \item In Phase II, there are two plausible worlds: 
    \begin{enumerate}
        \item if the expected probability the algorithm assigns  to expert $h_u$ in Phase I is at least $\En_{\sigma}\brk*{\sum_{t=1}^{\nicefrac{T}{2}}p_{h_u}^t}> \sqrt{\eps} \cdot T$
        then the adversary assigns negative labels for both groups
        \item else the adversary assigns positive labels to examples with group context $B$ while examples from group $A$ keep receiving positive and negative labels with probability equal to half.
    \end{enumerate}
    We will show that for any algorithm with vanishing approximate regret property, i.e. with $\apx_{\eps,T}(f)=o(T)$, the condition for the first world is never triggered and hence the above sequence is indeed balanced.
\end{enumerate}

We now show why this instance is unfair in composition with respect to false negative rate. The proof involves showing the following two claims:
\begin{enumerate}
    \item In Phase I, any $\epsilon$-approximate regret algorithm needs to select the negative expert $h_n$ most of the times to ensure small approximate regret with respect to $h_n$. This means that, in Phase I (where we encounter half of the positive examples from group $A$ and none from group $B$), the false negative rate of the algorithm is close to $1$.
    \item In Phase II, any $\epsilon$-approximate regret algorithm should quickly catch up to ensure small approximate regret with respect to $h_u$ and hence the false negative rate of the algorithm is closer to $\beta$. Since the algorithm is group-unaware, this creates a mismatch between the false negative rate of $B$ (that only receives false negatives in this phase) and $A$ (that has also received many false negatives before).
\end{enumerate}

\paragraph{Upper bound on probability of playing $h_u$ in Phase I.}
We now formalize the first claim by showing that any algorithm with $\En_{\sigma}\brk*{\sum_{t=1}^{\nicefrac{T}{2}}p_{h_u}^t}>\sqrt{\eps}\cdot T$ does not satisfy the approximate regret property. The algorithm suffers an expected loss of $\beta$ every time it selects expert $h_u$. On the other hand, when selecting expert $h_n$, it suffers a loss of $0$
 for members of group $B$ and an expected loss of $\nicefrac{1}{2}$ for members of group $A$. As a result, the expected loss of the algorithm in the first phase is:
 
\begin{align*}
 \En_\sigma\brk*{\sum_{t=1}^{\nicefrac{T}{2}}\sum_{f\in\mathcal{F}}p_f^t\cdot\ell_f^t}
 &= \En_\sigma\brk*{\sum_{t=1}^{\nicefrac{T}{2}}p_{h_u}^t}\cdot \beta+\En_\sigma\brk*{\sum_{t=1}^{\nicefrac{T}{2}}p_{h_n}^t\cdot\mathbf{1}_{g(t)=A}}\cdot \frac{1}{2}\\
 &=\En_\sigma\brk*{\sum_{t=1}^{\nicefrac{T}{2}}p_{h_u}^t}\cdot \beta+\prn*{\frac{T}{2}-\En_\sigma\brk*{\sum_{t=1}^{\nicefrac{T}{2}}p_{h_u}^t}}\cdot \frac{1}{4}\\
 &= \frac{T}{8}+\prn*{\beta-\frac{1}{4}}\cdot\En_\sigma\brk*{\sum_{t=1}^{\nicefrac{T}{2}}p_{h_u}^t}=\frac{T}{8}+\sqrt{\epsilon}\cdot \En_\sigma\brk*{\sum_{t=1}^{\nicefrac{T}{2}}p_{h_u}^t}
 \end{align*}
 
 In contrast, the negative expert has, in Phase I, expected loss of:
$$
\En_{\sigma}\brk*{\sum_{t=1}^{\nicefrac{T}{2}}\ell_{h_n}^t}=\frac{T}{8}.
$$
Therefore, if $\En_{\sigma}\brk*{\sum_{t=1}^{\nicefrac{T}{2}}p_{h_u}^t}>\sqrt{\eps}\cdot T$, the $\epsilon$-approximate regret of the algorithm with respect to $h_n$ is linear to the time-horizon $T$ (and therefore not vanishing) since:
\begin{align*}
 \En_{\sigma}\brk*{\sum_{t=1}^{\nicefrac{T}{2}}\sum_{f\in\mathcal{F}}p_f^t\cdot\ell_f^t-(1+\epsilon)\sum_{t=1}^{\nicefrac{T}{2}}\ell_{h_N}^t}&=\frac{T}{8}+\sqrt{\epsilon}\cdot\En_\sigma\brk*{\sum_{t=1}^{\nicefrac{T}{2}}p_{h_u}^t}-(1+\eps)\frac{T}{8}> \frac{7\eps}{8}\cdot T.
\end{align*}

\paragraph{Upper bound on probability of playing $h_n$ in Phase II.}
Regarding the second claim, we first show that $\En_{\sigma}\brk*{\sum_{t=T/2+1}^T p_{h_n}^t}\leq 16\sqrt{\epsilon}\cdot T$ for any $\epsilon$-approximate regret algorithm with $\epsilon<\nicefrac{1}{16}$.

The expected loss of the algorithm in the second phase is: 
\begin{align*}
\En_{\sigma}\brk*{\sum_{t=\nicefrac{T}{2}+1}^T \sum_{f\in\mathcal{F}}p_f^t\ell_f^t}&=\En_\sigma\brk*{\sum_{t=\nicefrac{T}{2}+1}^T p_{h_n}^t}\cdot \frac{3}{4}+\prn*{\frac{T}{2}-\En_\sigma\brk*{\sum_{t=\nicefrac{T}{2}+1}^T p_{h_n}^t}}\cdot \beta.
\end{align*}

Since, in Phase I, the best case scenario for the algorithm is to always select expert $h_n$ and incur a loss of $\nicefrac{T}{8}$, this implies that for $\epsilon<\nicefrac{1}{16}$:
\begin{align*}\En_\sigma\brk*{\sum_{t=1}^T \sum_{f\in\mathcal{F}}p_f^t\ell_f^t}&\geq \frac{T}{8}+\frac{T}{2}\cdot \beta +\En_\sigma\brk*{\sum_{t={\nicefrac{T}{2}+1}}^Tp_{h_n}^t}\cdot \prn*{\frac{3}{4}-\beta}\\&=\frac{\prn*{1+2\sqrt{\epsilon}}\cdot T}{4}+\En_\sigma\brk*{\sum_{t={\nicefrac{T}{2}+1}}^Tp_{h_n}^t}\cdot \prn*{\frac{1}{2}-\sqrt{\epsilon}}
\\&> \frac{T}{4}+\En_{\sigma}\brk*{\sum_{t=\nicefrac{T}{2}+1}^Tp_{h_n}^t}\cdot \frac{1}{4}.
\end{align*}

On the other hand, the cumulative expected loss of the $\beta$-inaccurate expert $h_u$ is
$$\En\brk*{\sum_{t=1}^T \ell_{h_u}^t}=\beta\cdot T=\frac{T}{4}+\sqrt{\epsilon}\cdot T.$$ 

Therefore, if the algorithm has $\En_\sigma\brk*{\sum_{t=\nicefrac{T}{2}+1}^Tp_{h_n}^t}>16\sqrt{\epsilon} \cdot T$, the $\epsilon$-approximate regret of the algorithm with respect to $h_u$ is linear to the time-horizon since $\epsilon\le 1$, we have:
\begin{align*}
   \En_{\sigma}\brk*{\sum_{t=1}^T \sum_{f\in\mathcal{F}}p_f^t\ell_f^t -(1+\epsilon)\sum_{t=1}^T\ell_{h_u}^t}&>\prn*{\frac{T}{4} +\En_\sigma\brk*{\sum_{t=\nicefrac{T}{2}+1}^Tp_{h_n}^t}\cdot \frac{1}{4}}-(1+\epsilon)\cdot \prn*{\frac{T}{4}+\sqrt{\epsilon}\cdot T} 
   \\&\geq \En_\sigma\brk*{\sum_{t=\nicefrac{T}{2}+1}^Tp_{h_n}^t}\cdot \frac{1}{4}-3\sqrt{\epsilon}\cdot T > \sqrt{\epsilon} \cdot T
\end{align*}
The last inequality holds since $\eps\cdot \nicefrac{T}{4}+\eps\cdot\sqrt{\eps}\cdot T+\sqrt{\eps}\cdot T\leq 3\sqrt{\eps}\cdot T$ for $\eps\leq 1$.

Thus, we have shown that for, for $\epsilon<\nicefrac{1}{16}$, any algorithm with vanishing approximate regret, necessarily we have $\En_\sigma\brk*{\sum_{t=\nicefrac{T}{2}+1}^Tp_{h_n}^t}\le 16\sqrt{\epsilon} \cdot T$. 

\paragraph{Gap in false negative rates between groups $A$ and $B$.}
We now compute the expected false negative rates for the two groups, assuming that $\epsilon<\nicefrac{1}{16}$. Since we focus on algorithms that satisfy the vanishing regret property, we have already established that:
\begin{equation}
    \En_\sigma\brk*{\sum_{t=1}^{\nicefrac{T}{2}}p_{h_u}^t}\leq \sqrt{\eps}\cdot T \quad \text{ and }\quad
    \En_\sigma\brk*{\sum_{t=\nicefrac{T}{2}+1}^Tp_{h_n}^t}\leq 16\sqrt{\epsilon}\cdot T.
\end{equation}

For ease of notation, let  $G_{A,+}^t=\mathbf{1}\crl*{g(t)=A,y(t)=+}$ and $G_{B,+}^t=\mathbf{1}\crl*{g(t)=B,y(t)=+}$. Since the group context at round $t$ arrives independent of $\sigma^{1:t-1}$ and the adversary is oblivious, we have that $G_{A,+}^t,G_{B,+}^t$ are independent of $\sigma^{1:t-1}$, and hence also independent of $p_{h_u}^t,p_{h_n}^t$. 

Since the algorithm is group-unaware, the expected cumulative probability that the algorithm uses $h_n$ in Phase II is the same for both groups. We combine this with the facts that under the online learning protocol with group context, examples of group $B$ arrive stochastically with probability half but only receive positive labels in Phase II, we obtain:
\begin{equation}
\En_\sigma\brk*{\sum_{t=\nicefrac{T}{2}+1}^T p_{h_n}^t\cdot G_{B,+}^t}=\frac{1}{2}\cdot  \En_\sigma\brk*{\sum_{t=\nicefrac{T}{2}+1}^Tp_{h_n}^t} \leq 8\sqrt{\epsilon}\cdot T.
\end{equation}

Recall that group $B$ in Phase I has no positive labels, hence the false negative rate on group $B$ is:
\begin{align*}
    \En_{\sigma}\brk*{FNR_{\mathcal{L}}^\sigma(B)}&=\En_\sigma\brk*{\frac{\sum_{t=\nicefrac{T}{2}+1}^T G_{B,+}^t\cdot\prn*{p_{h_u}^t\cdot \beta+p_{h_n}^t\cdot 1} }{\sum_{t=\nicefrac{T}{2}+1}^T\cdot G_{B,+}^t}}\\
    &=\beta+\En_\sigma\brk*{\frac{(1-\beta)\cdot\sum_{t=\nicefrac{T}{2}+1}^T G_{B,+}^t\cdot p_{h_n}^t }{\sum_{t=\nicefrac{T}{2}+1}^T G_{B,+}^t}}
\end{align*}

In order to upper bound the above false negative rate, we denote the following good event by $$\mathcal{E}^B(\eta)=\crl*{\sigma^{1:T}:\sum_{t={\nicefrac{T}{2}+1}}^T G_{B,+}^t> (1-\eta)\En\brk*{\sum_{t=\nicefrac{T}{2}+1}^T G_{B,+}^t}}.$$ 
By Chernoff bound, the probability of the bad event is:
$$
\mathbb{P}\brk*{\neg\mathcal{E}^B(\eta)}=\exp\prn*{-\frac{\eta^2 \cdot \En\brk*{\sum_{t=\nicefrac{T}{2}+1}^T G_{B,+}^t}}{2}}.
$$
For $\eta^B= \sqrt{\nicefrac{16\log(T)}{T}}$, this implies that $\mathbb{P}\brk{\neg \mathcal{E}^B(\eta^B)}\leq \nicefrac{1}{T^2}$ since $\En_\sigma\brk{\sum_{t=\nicefrac{T}{2}+1}^TG_{B,+}^t}=\nicefrac{T}{4}$.

Therefore, by first using the bound on $\sum_{t=\nicefrac{T}{2}+1}^T G_{B,+}^t$ on the good event and the bound on the probability of the bad event, and then taking the limit $T\rightarrow \infty$, it holds that:
\begin{align*}
    \En_\sigma\brk*{ FNR_{\mathcal{L}}^\sigma(B)}&=\beta+\En_\sigma\brk*{\frac{(1-\beta)\cdot\sum_{t=\nicefrac{T}{2}+1}^T G_{B,+}^t\cdot p_{h_n}^t}{\sum_{t=\nicefrac{T}{2}+1}^T G_{B,+}^t}}\\
    &\leq \beta+\frac{1-\beta}{1-\eta^B}\cdot \frac{8\sqrt{\eps}\cdot T}{T/4}\cdot \mathbb{P}\brk*{\mathcal{E}^B(\eta^B)}+1\cdot \mathbb{P}\brk*{\neg\mathcal{E}^B(\eta^B)} \\
    &\leq \beta  +\frac{32\sqrt{\epsilon}}{1-\eta^B}+\frac{1}{T^2}\rightarrow \frac{1}{4}+33\sqrt{\eps}.
\end{align*}

We now move to the false negative rate of $A$:
\begin{align*}
    \En_{\sigma}\brk*{FNR_{\mathcal{L}}^\sigma(A)}&=\En_\sigma\brk*{\frac{\sum_{t=1}^T G_{A,+}^t\cdot\prn*{p_{h_u}^t\cdot \beta+p_{h_n}^t\cdot 1} }{\sum_{t=1}^T G_{A,+}^t}}.
\end{align*}
Similarly as before, letting $\mathcal{E}^A(\eta)=\crl*{\sigma^{1:T}:\sum_{t=1}^TG_{A,+}^t < (1+\eta)\En\brk*{\sum_{t=1}^T G_{A,+}^t}}$ and, since $\mathbb{P}\brk{\neg \mathcal{E}^A(\eta)}=\exp\prn*{-\nicefrac{\eta^2\cdot \En\brk*{\sum_{t=1}^T G_{A,+}^t}}{3}}$, we obtain that, for $\eta^A=\sqrt{\nicefrac{24\log(T)}{T}}$, $\mathbb{P}\brk{\neg \mathcal{E}^A(\eta^A)}=\nicefrac{1}{T^2}$.

Recall that for our instance $\En_{\sigma}\brk*{G_{A,+}^t}=\nicefrac{T}{4}$ and $G_{A,+}^t$ is independent of $p_{h_u}^t$. From our previous analysis we also know that:
\begin{equation}
    \En_{\sigma}\brk*{\sum_{t=1}^{\nicefrac{T}{2}}p_{h_u}^t G_{A,+}^t}\leq \frac{\sqrt{\epsilon}\cdot T}{4} \quad \text{ and }\quad
    \En_{\sigma}\brk*{\sum_{t=\nicefrac{T}{2}+1}^Tp_{h_u}^t G_{A,+}^t} \leq \frac{T}{8}
    \label{eq:3}
\end{equation}
As a result, using that $\En\brk*{\sum_{t=1}^{\nicefrac{T}{2}}G_{A,+}^t}=\En\brk*{\sum_{t=\nicefrac{T}{2}+1}^T G_{A,+}^t}=\frac{T}{8}$ and Inequalities  \eqref{eq:3}, we obtain:
\begin{align*}
\En_\sigma\brk*{\sum_{t=1}^T G_{A,+}^t\cdot\prn*{p_{h_u}^t\cdot \beta+p_{h_n}^t\cdot 1}}&=\En_\sigma\brk*{\sum_{t=1}^T G_{A,+}^t \cdot - \sum_{t=1}^T G_{A,+}^t\cdot p_{h_u}^t(1-\beta) } \\
&\geq \frac{T}{4}\prn*{1-(1-\beta)\cdot \prn{\frac{1}{2}+\sqrt{\epsilon}}}.
\end{align*}
Therefore, similarly with before, it holds that:
\begin{align*}
    \En_\sigma\brk*{FNR_{\mathcal{L}}^{\sigma}(A)}&=\En_\sigma\brk*{\frac{\sum_{t=1}^T G_{A,+}^t\cdot\prn*{p_{h_u}^t\cdot \beta+p_{h_n}^t\cdot 1} }{\sum_{t=1}^T G_{A,+}^t}}\\
    &\geq \frac{1-(1-\beta)\cdot(\frac{1}{2}+\sqrt{\epsilon})}{(1+\eta^A)}\cdot \mathbb{P}\brk*{\mathcal{E}^A(\eta^A)}+0\cdot \mathbb{P}\brk*{\neg \mathcal{E}^A(\eta^A)}\\
    &\geq \frac{\frac{1}{2}-\sqrt{\epsilon}+\nicefrac{\beta}{2}}{1+\eta^A}\prn*{1-\frac{1}{T^2}}> \frac{\frac{1}{2}-\sqrt{\epsilon}+\nicefrac{1}{8}}{1+\eta^A}\prn*{1-\frac{1}{T^2}}\rightarrow \frac{5}{8}-\sqrt{\epsilon}.
\end{align*}
As a result, the difference between the false negative rate in group $A$ and the one at group $B$ is $\nicefrac{3}{8}+34\sqrt{\epsilon}$ which can go arbitrarily close to \nicefrac{3}{8} by appropriately selecting $\epsilon$ to be small enough, for any vanishing approximate regret algorithm. This concludes the proof.
\end{proof}

\subsection{Group-aware algorithms}
We now turn our attention to group-aware algorithms, that can use the group context of the example to select the probability of each expert and provide a similar impossibility result. There are three changes compared to the impossibility result we provided for group-unaware algorithms. First, the adversary is not oblivious but instead is adaptive. Second, we do not have perfect balance across populations but instead require that the minority population arrives with probability $b<0.49$, while the majority population arrives with probability $1-b$. Third, the labels are not equally distributed across positive and negative for each population but instead positive labels for one group are at least a $c$ percentage of the total examples of the group for a small $c>0$. Although the upper bounds on $b$ and $c$ are not optimized, our impossibility result cannot extend to $b=c=\nicefrac{1}{2}$. Understanding whether one can achieve fairness in composition for some values of $b$ and $c$ is an interesting open question. Our impossibility guarantee is the following:

\begin{theorem}\label{thm:eq_odds_negative_imbalance}
\vspace{0.5em}
For any group imbalance $b<0.49$ and $0<\alpha<\frac{0.49-0.99b}{1-b}$, there exists $\epsilon_0>0$ such that for all $0<\eps<\epsilon_0$ any algorithm that satisfies $\En_{\sigma}\brk*{\apx_{\eps,T}(f)}=o(T)$ for all $f\in\mathcal{F}$, is $\alpha$-unfair in composition. 
\end{theorem}

\begin{proof}
The instance has two groups: $\mathcal{G}=\crl{A,B}$. Examples with group context $A$ are discriminated against and arrive randomly with probability $b<\nicefrac{1}{2}$ while examples with group context $B$ are boosted by the discrimination and arrive with the remaining probability $1-b$.
There are again two experts $\mathcal{F}=\crl{h_n,h_p}$, which output score values in $\hat{\mathcal{Y}}=[0,1]$, where $\hat{y}_f^t$ can be interpreted as the probability that expert $f$ assigns to label being $+$ in round $t$. We use the earlier loss function of $\ell(\hat{y},y)=\hat{y}\cdot\mathbf{1}\{y=-\}+(1-\hat{y})\cdot\mathbf{1}\{y=+\}$.  The first expert $h_n$ is again pessimistic  and always predicts negative, i.e. $\hat{y}^t_{h_n}=0$, while the other expert $h_p$ is optimistic and always predicts positive,  i.e. $\hat{y}^t_{h_p}=1$. These experts again satisfy fairness in isolation with respect to equalized odds (false negative rate and false positive rate). Let $c=\nicefrac{1}{101^2}$ denote the percentage of the input that is about positive examples for $A$, ensuring that $|\mathcal{S}_g^{\sigma}(FNR)|=\Omega(T)$. The instance proceeds in two phases. 
\begin{enumerate}
    \item Phase I lasts $\Theta\cdot T$ rounds for $\Theta=\tledit{101}c$. The adversary assigns negative labels on examples with group context $B$. For examples with group context $A$, the adversary acts as following:
    \begin{itemize}
        \item if the algorithm assigns probability on the negative expert below   $\gamma(\epsilon)=\frac{99-2\epsilon}{100}$, i.e. $p_{h_n}^t(\sigma^{1:t-1})< \gamma(\epsilon)$, then the adversary assigns negative label.
        \item otherwise, the adversary assigns positive labels. 
    \end{itemize}
    \item In Phase II, there are two plausible worlds:
    \begin{enumerate}
        \item the adversary assigns negative labels to both groups if the expected number of times that the algorithm selected the negative expert with probability higher than $\gamma(\epsilon)$ on members of group $A$ is less than $c\cdot b\cdot T$, i.e. $\En_{\sigma}\brk*{\mathbf{1}\crl*{t\leq \Theta\cdot T: g(t)=A, p_{h_n}^t\geq \gamma(\epsilon)}}< c\cdot b\cdot T$. 
        \item otherwise she assigns positive labels to examples with group context $B$ and negative labels to examples with group context $A$.
    \end{enumerate}
    Note that, as before, the condition for the first world will never be triggered by any no-regret learning algorithm (we elaborate on that below) which ensures that $\En_{\sigma}{|S_A^\sigma(FNR)|}\geq c\cdot b\cdot T$.
\end{enumerate}

The proof is based on the following claims:
\begin{enumerate}
    \item In Phase I, any vanishing approximate regret algorithm enters the second world of Phase II. 
    \item This implies  a lower bound on the false negative rate on $A$, i.e. $FNR(A)\geq \gamma(\epsilon)=\frac{99-2\epsilon}{100}$. 
    \item In Phase II, any $\epsilon$-approximate regret algorithm assigns large enough probability to the positive expert $h_p$ for group $B$. This implies an upper bound on the false negative rate on $B$, i.e. $FNR(B)\leq \nicefrac{1}{2(1-b)}$. Therefore this provides a gap in the false negative rates of at least $\alpha$.
\end{enumerate}

\paragraph{Proof of first claim.}
To prove the first claim, we apply the method of contradiction. Assume that the algorithm has $\En_{\sigma}\brk*{\mathbf{1}\crl*{t\leq \Theta\cdot T: g(t)=A, p_{h_n}^t\geq \gamma(\epsilon)}}< c\cdot b\cdot T$. This means that the algorithm faces an expectation of at least $(\Theta-c)\cdot b\cdot T$ negative examples, while predicting the negative expert with probability at most $\gamma(\epsilon)=\frac{99-2\epsilon}{100}$,thereby making an error with probability $1-\gamma(\epsilon)$. Therefore the expected loss of the algorithm is at least:
$$
\En_\sigma\brk*{\sum_{t=1}^{\Theta\cdot T}\sum_{f\in\mathcal{F}}p_f^t\cdot\ell_f^t}> \prn*{\Theta-c}\cdot b\cdot T \cdot (1-\gamma(\epsilon))=c\cdot b\cdot (1+2\epsilon)\cdot T.
$$
At the same time, expert  $h_n$ makes at most $c\cdot b\cdot T$ errors (at the positive examples)
$$
\En_{\sigma}\brk*{\sum_{t=1}^T \ell_{h_n}^t} < c\cdot b\cdot T.
$$
Therefore, if $\En_{\sigma}\brk*{\mathbf{1}\{t\leq \Theta\cdot T: g(t)=A, p_{h_n}^t\ge f(\epsilon)\}}< c\cdot b\cdot T$, the $\epsilon$-approximate regret of the algorithm with respect to $h_n$ is linear to the time-horizon $T$ (and therefore not vanishing) since:
$$
\En_{\sigma}\brk*{\sum_{t=1}^T \sum_{f\in\mathcal{F}}p_f^t\ell_f^t-(1+\epsilon)\sum_{t=1}^T\ell_{h_n}^t}> \epsilon \cdot c\cdot b\cdot T.
$$
This violates the vanishing approximate regret property, thereby leading to contradiction.

\paragraph{Proof of second claim.} The second claim follows directly by the above construction, since positive examples only appear in Phase I when the probability of error on them is greater than $\gamma(\epsilon)$.

\paragraph{Proof of third claim.} Having established that any vanishing approximate regret algorithm will always enter the second world, we now focus on the expected loss of expert $h_p$ in this case. This expert makes errors at most in all Phase I and in the examples of Phase II with group context $A$:

$$
\En_{\sigma}\brk*{\sum_{t=1}^T \ell_{h_p}^t}\leq \Theta\cdot T+ b \cdot (1-\Theta)\cdot T \le\Theta\cdot T+ 0.49 \cdot (1-\Theta)\cdot T $$
 
Since group $B$ has only positive examples in Phase II, the expected loss of the algorithm is at least:
$$
\En_{\sigma}\brk*{\sum_{t=1}^T\sum_{f\in\mathcal{F}}p_f^t\ell_f^t}\geq \En_\sigma \brk*{\sum_{t=\Theta\cdot T+1}^T p_{h_n}^t\cdot \mathbf{1}_{g(t)=B}}
$$

We now show that $\En_\sigma \brk*{\sum_{t=\Theta\cdot T+1}^T p_{h_n}^t\cdot \mathbf{1}_{g(t)=B}}\leq 
(\nicefrac{1}{2}+\epsilon)\cdot (1-\Theta)\cdot T$.
If this is not the case, then the algorithm does not have vanishing $\epsilon$-approximate regret with respect to expert $h_p$ since:
\begin{align*}
    \En_{\sigma}\brk*{\sum_{t=1}^T\sum_{f\in\mathcal{F}}p_f^t\ell_f^t-(1+\epsilon)\sum_{t=1}^T \ell_{h_p}^t}&> \prn*{\frac{1}{2}+\epsilon}\cdot (1-\Theta) T-(1+\epsilon)\cdot 0.49\cdot(1-\Theta)T- (1+\epsilon) \Theta T\\
    &\geq \prn*{\frac{1}{2}+\epsilon-0.49-0.49\epsilon}\cdot (1-\Theta)\cdot T-(1+\epsilon)\cdot \Theta\cdot T
    \\ &> (0.01+0.51\epsilon)\cdot \frac{100}{101}\cdot T-\frac{1+\epsilon}{101}\cdot T\geq  \frac{50}{101}\epsilon\cdot T
\end{align*}
Given the above, we now establish a gap in the fairness with respect to false negative rate. Since group $A$ only experiences positive examples when expert $h_n$ is offered probability higher than $\gamma(\epsilon)=\frac{99-2\epsilon}{100}$, 
this means that:
$$
\En_{\sigma}\brk*{FNR_{\mathcal{L}}^{\sigma}(A)}\rightarrow 0.99-0.02\epsilon
$$
Regarding group $B$, we need to take into account the low-probability event that the actual realization has significantly less than $(1-b)(1-\Theta)\cdot T$ examples of group $B$ in Phase II (all are positive examples). This can be handled via similar Chernoff bounds as in the proof of the previous theorem. As a result, the expected false negative rate at group $B$ is:
$$
\En_{\sigma}\brk*{FNR_{\mathcal{L}}^{\sigma}(B)}\rightarrow \frac{\En_\sigma \brk*{\sum_{t=\Theta\cdot T+1}^T p_{h_n}^t\cdot \mathbf{1}_{g(t)=B}}}{\En_\sigma \brk*{\sum_{t=\Theta\cdot T+1}^T
\cdot\mathbf{1}_{g(t)=B}}}\le\frac{\prn*{\nicefrac{1}{2}+\epsilon}\cdot(1-\Theta)\cdot T}{(1-b)\cdot (1-\Theta)\cdot T}= \frac{\nicefrac{1}{2}+\epsilon}{1-b}
$$
which establishes a gap in the fairness with respect to false negative rate of $\alpha\rightarrow \frac{0.49-0.99b}{1-b}$ selecting $\epsilon>0$ appropriately small.
\end{proof}

\section{Fairness in composition with respect to an alternative metric}
\label{sec:equalized_error_rates}
The negative results of the previous section give rise to a natural question of whether fairness in composition can be achieved for some other fairness metric in an online setting. 

We answer this question positively by suggesting the \emph{equalized error rates} metric $EER$ which captures the average loss over the total number of examples (independent of whether this loss comes from false negative or false positive examples). The relevant subset induced by this metric $\mathcal{S}_g^{\sigma}(EER)$ is the set of all examples coming from group $g\in\mathcal{G}$. We again assume that experts are fair in isolation with respect to equalized error rate and show that a simple scheme where we run separately one instance of multiplicative weights for each group achieves fairness in composition (Theorem~\ref{thm:eq_error_rates_positive}). The result holds for general loss functions (beyond pure classification) and is robust to the experts only being approximately fair in isolation. A crucial property we use is that multiplicative weights not only does not perform worse than the best expert; it also does not perform better. In fact, this property holds more generally by online learning algorithms with optimal regret guarantees \tledit{\cite{gofer2016lower}.}
\tlcomment{Wrong citation: Correct citation: \cite{gofer2016lower}}

Interestingly, not all algorithms can achieve fairness in composition even with respect to this refined notion. We provide two algorithm classes where this is unachievable. First, we show that any algorithm (subject to a technical condition satisfied by algorithms such as multiplicative weights and follow the perturbed leader) that ignores the group identity can be unboundedly unfair with respect to equalized error rates (Theorem \ref{thm:single_algo_error_rates}). This suggests that the algorithm needs to actively discriminate based on the groups to achieve fairness with respect to equalized error rates. Second, we show a similar negative statement for any algorithm that satisfies the more involved guarantee of small shifting regret, therefore outperforming the best expert (Theorem \ref{thm:shifting_algo_error_rates}). This suggests that the algorithm used should be good but not too good. This result is, to the best of our knowledge, a first application where shifting regret may not be desirable which may be of independent interest. 

\subsection{The positive result}
We run separate instances of multiplicative weights with a fixed learning rate $\eta$, one for each group. More formally, for each pair of expert $f\in\mathcal{F}$ and group $g\in\mathcal{G}$, we initialize weights $w_{f,g}^1=1$. At round $t=\crl{1,2,\dots,T}$, an example with group context $g(t)$ arrives and the learner selects a probability distribution based to the corresponding weights: $p^t_f=\frac{w_{f,g(t)}^t}{\sum_{j\in\mathcal{F}}w_{j,g(t)}^t}$. Then the weights corresponding to group $g(t)$ are updated exponentially: $w_{f,g}^{t+1}=w_{f,g}^t\cdot (1-\eta)^{\ell_f^t\cdot \textbf{1}\crl{g(t)=g}}$.

\begin{theorem}
\label{thm:eq_error_rates_positive}
\vspace{0.3em}
For any $\alpha>0$ and any $\epsilon<\alpha$ such that running separate instances of multiplicative weights for each group with learning rate $\eta=\min(\epsilon,\nicefrac{\alpha}{6})$ guarantees $\alpha$-fairness in composition and $\epsilon$-approximate regret of at most $O(\nicefrac{|\mathcal{G}|\log(d)}{\epsilon})$.
\end{theorem}
\begin{proof}
The proof is based on the property that multiplicative weights performs not only no worse than the best expert in hindsight but also no better. Therefore the average performance of multiplicative weights at each group is approximately equal to the average performance of the best expert in that group. Since the experts are  fair in isolation, the average performance of the best expert in all groups is the same which guarantees the equalized error rates desideratum. We make these arguments formal below.

We follow the classical potential function analysis of multiplicative weights but apply bidirectional bounds to also lower bound the performance of the algorithm by the performance of the comparator. For each group $g\in\mathcal{G}$ and every expert $f\in\mathcal{F}$, let $L_{f,g}=\sum_{t:g(t)=g}\ell_f^t\cdot \mathbf{1}\crl{g(t)=g}$ be the cumulative loss of expert $f$ in examples with group context $g$, and $\hat{L}_g=\sum_{t=1}^T\sum_{f\in\mathcal{F}}p_f^t\ell_f^t\cdot\mathbf{1}\crl{g(t)=g}$ to denote the expected loss of the algorithm on these examples. We also denote the best in hindsight expert on these examples by $f^{\star}(g)=\argmin_{f\in \mathcal{F}}L_{f,g}$. Recall that $w_{f,g}^t=(1-\eta)^{\sum_{\tau\leq t: g(\tau)=g}\ell_f^{\tau}}$ is the weight of expert $f$ in the instance of group $g$ and let $W_{t,g}=\sum_{j\in\mathcal{F}}w_{j,g}^t$ be its potential function. 

To show that the algorithm does not perform much worse than any expert, we follow the classical potential function analysis and, since $(1-\eta)^x\leq 1-\eta x$ for all $x\in[0,1]$ and $\eta\leq 1$, we obtain:
\begin{align*}
    W_{t+1,g}=\sum_{j\in\mathcal{F}}w_{j,g}^t\cdot (1-\eta)^{\ell_j^t\cdot \mathbf{1}\crl{g(t)=g}}&\leq \sum_{j\in\mathcal{F}} w_{j,g}^t\cdot (1-\eta \ell_j^t\cdot\mathbf{1}\crl{g(t)=g})=W_{t,g}\cdot \prn{1-\eta\sum_{j\in\mathcal{F}}p_j^t\ell_j^t}.
\end{align*}
By the classical analysis, for all $f\in\mathcal{F}$ and $g\in\mathcal{G}$:
\begin{align*}
    w_{f,g}^{T+1}=(1-\eta)^{\sum_{t=1}^T \ell_f^t\cdot\mathbf{1}\crl{g^t=g}}\leq W_{T+1,g} \leq d\cdot \prod_{t=1}^T (1-\eta \sum_{j\in\mathcal{F}} p_{j}^t\ell_j^t\cdot \mathbf{1}\crl{g(t)=g})
\end{align*}
where the left inequality follows from the fact that all summands of $W_{T+1,g}$ are positive and the right inequality follows by unrolling $W_{T+1,g}$ and using that $W_{1,g}=d$.

Taking logarithms and using that $-\eta-\eta^2<\ln(1-\eta)<-\eta$ for $\eta<\nicefrac{1}{2}$, this implies that for all $f\in\mathcal{F}$:
\begin{equation}\label{eq:mwu_classical}
    \hat{L}_g \leq (1+\eta)\cdot L_{f,g}+\frac{\ln(d)}{\eta}
\end{equation}

We now use the converse side of the inequalities to show that multiplicative weights also does not perform much better than the best expert in hindsight $f^{\star}(g)$. Using that $(1-\eta)^x\geq 1-\eta(1+\eta) x$ for all $x\in[0,1]$, we obtain:
\begin{align*}
    W_{t+1,g}=\sum_{j\in\mathcal{F}}w_{j,g}^t\cdot (1-\eta)^{\ell_j^t\cdot \mathbf{1}\crl{g(t)=g}}&\geq \sum_{j\in\mathcal{F}} w_{j,g}^t\cdot \prn*{1-\eta(1+\eta)\cdot \ell_j^t\cdot\mathbf{1}\crl{g(t)=g}}\\&=W_{t,g}\cdot \prn*{1-\eta(1+\eta)\sum_{j\in\mathcal{F}}p_i^t\ell_i^t}.
\end{align*}
Using that $f^{\star}(g)$ is the best expert in hindsight, we can upper bound $\sum_{j\in\mathcal{F}}w_{j,g}^t\leq d\cdot \max_{j\in\mathcal{F}}w_{j,g}^t=d\cdot\max_{f\in \mathcal{F}}(1-\eta)^{\sum_{t=1}^t \ell_{f}^{t}\mathbf{1}\crl{g^t=g}} $. Similarly to before, it therefore follows that:
\begin{align*}
    d\cdot(1-\eta)^{\sum_{t=1}^T \ell_{f^{\star}(g)}^t\mathbf{1}\crl{g^t=g}}\geq W_{T+1} \geq d\cdot \prod_{t=1}^{T} \prn*{1-\eta(1+\eta) \sum_{j\in\mathcal{F}} p_j^t\ell_j^t}
\end{align*}
which, for $\eta<\nicefrac{1}{2}$, implies that:
\begin{equation}\label{eq:mwu_opposite_direction}
    \widehat{L}_g \geq (1-4\eta)\cdot L_{f^\star(g),g}
\end{equation}
The expected $\epsilon$-approximate regret of this algorithm is at most $6\cdot|\mathcal{G}|$ times the one of a single multiplicative weights instance (by summing over inequalities \eqref{eq:mwu_classical} for all $g\in\mathcal{G}$ and since $\nicefrac{\epsilon}{6}\leq \eta\leq \epsilon$). What is left to show is that the $\alpha$-fairness in composition guarantee is satisfied, that is there exists $T_0$ (function of $\alpha$ and $\epsilon$) such that when the number of examples from each group is at least $T_0$, the maximum difference between average expected losses across groups is bounded by $\alpha$. Let $g^{\star}$ be the group with the smallest average expected loss. We will show that the maximum difference from the average expected loss of any other group $g$ is at most $\alpha$ for $T_0=\nicefrac{6\ln(d)}{\eta\alpha}$. Since the experts are fair in isolation, we know that $\frac{L_{f,g}}{|\crl{t:g^t=g}|}=\frac{L_{f,g'}}{|\crl{t:g^t=g'}|}$ for all $f\in\mathcal{F}$ and $g,g'\in\mathcal{G}$. Combining this with inequalities \eqref{eq:mwu_classical} and \eqref{eq:mwu_opposite_direction} and the fact that the losses are in $[0,1]$ and $\eta\leq \nicefrac{\alpha}{6}$, we obtain:

\begin{align*}
\frac{\widehat{L}_g}{|\crl{t:g(t)=g}|}-\frac{\widehat{L}_{g^{\star}}}{|\crl{t:g(t)=g^{\star}}|}&\leq  \frac{(1+\eta)\cdot L_{f^{\star}(g),g}}{|\crl{t:g(t)=g}|}+\frac{\ln(d)}{\eta\cdot |\crl{t:g(t)=g}|}- \frac{(1-4\eta)\cdot L_{f^{\star}(g^{\star}),g^{\star}}}{|\crl{t:g(t)=g^{\star}}|}
\\&\leq 5\eta\cdot\frac{ L_{f^{\star}(g^{\star}),g^{\star}}}{|\crl{t:g(t)=g^{\star}}|}+\frac{\ln(d)}{\eta\cdot T_0}\leq\alpha.
\end{align*}
\end{proof}
\begin{remark}
If the instance is instead only approximately fair in isolation with respect to equalized error rates, i.e. the error rates of the two experts are not exactly equal but within some constant $\kappa$, the same analysis implies $(\alpha+\kappa)$-fairness in composition with respect to equalized error rates.
\end{remark}

\subsection{Impossibility results}
\paragraph{Group-unaware algorithms.} In the previous algorithm, it was crucial that the examples of the one group do not interfere with the decisions of the algorithm on the other group. We show that, had we run one multiplicative weights algorithm in a group-unaware way, we would not have accomplished fairness in composition. In fact, this impossibility result holds for any algorithm with vanishing $\epsilon$-approximate regret where the learning dynamic ($p^t$ at each round $t$) is a deterministic function of the difference between the cumumative losses of the experts (without taking into consideration their identity). This is satisfied, for instance by multiplicative weights and follow the perturbed leader with a constant learning rate.
Unlike previous section, the impossibility results for equalized error rates require groups to arrive adversarially (which also occurs in the above positive result). \begin{theorem}\label{thm:single_algo_error_rates}
\vspace{0.3em}
For any $\alpha>0$ and for any $\epsilon>0$, running a single algorithm from the above class in a group-unaware way is $\alpha$-unfair in composition with respect to equalized error rate. 
\end{theorem}
\begin{proof}
The instance has two groups $\mathcal{G}=\crl{A,B}$ that come in an adversarial order, two experts $\mathcal{F}=\crl{f_1,f_2}$, and consists of four phases of equal size. At each phase one predictor is always correct and the other one always incorrect. 
\begin{enumerate}
    \item In Phase I $\crl{1,\dots,\nicefrac{T}{4}}$, examples of group $A$ arrive and the first predictor is correct: $\ell_{f_1}^t=0$ and $\ell_{f_2}^t=1$.
    \item In Phase II $\crl{\nicefrac{T}{4}+1,\dots,\nicefrac{T}{2}}$, examples of group $B$ arrive and the second predictor is the correct one: $\ell_{f_1}^t=1$ and $\ell_{f_2}^t=0$.
    \item In Phase III $\crl{\nicefrac{T}{2},\dots,\nicefrac{3T}{4}+1}$, examples of group $A$ arrive and the second predictor is again the better one, i.e. $\ell_{f_1}^t=1$ and $\ell_{f_2}^t=0$.
    \item Finally, in Phase IV $\crl{\nicefrac{3T}{4}+1,\dots,T}$, examples of group $B$ arrive and now the first predictor is accurate: $\ell_{f_1}^t=0$ and $\ell_{f_2}^t=1$.
\end{enumerate} Note that both experts are fair in isolation with respect to equalized error rates as they both have $50\%$ error rate in each group.

Since the loss of the first expert is $0$ in the first quarter of the setting: $\sum_{t=1}^{\nicefrac{T}{4}}\ell_{f_1}=0$, any algorithm with vanishing approximate regret needs to have sublinear loss during this quarter to be robust against an adversary that continues giving $0$ losses to $f_1$. Therefore, in particular, it holds that:
$$
\sum_{t=1}^{\nicefrac{T}{4}}p_{f_2}< \frac{1-\alpha}{8}{T}.
$$
As a result, the error rate on group $A$ is at most $EER(A)\leq \frac{1-\alpha}{2}$ in Phase I and, since the algorithm's distribution is deterministic based on the difference in the losses, this also applies to Phase III.

Regarding group $B$, note that any time $t$ where an example of group $B$ arrives has a $1-1$ mapping to the time $t-\nicefrac{T}{4}$ where a member from group $A$ came, where the predictions of the algorithm are the same since the difference between losses are the same. Therefore, by our assumption on the dynamic, the cumulative probability the correct expert is upper bounded by $\frac{1-\alpha}{2}$ which implies that group $B$ incurs an equalized error rate of $EER(B)\geq \frac{1+\alpha}{2}$. Thi concludes the proof.
\end{proof}

\paragraph{Shifting algorithms.} The reader may be also wondering whether it suffices to just run separate learning algorithms in the two groups or whether multiplicative weights has a special property. In the following theorem, we show that the latter is the case. In particular, multiplicative weights has the property of not doing better than the best expert in hindsight. The main representative of algorithms that do not have such a property are the algorithms that achieve low approximate regret compared to a shifting benchmark (tracking the best expert). More formally, approximate regret against a shifting comparator $f^{\star}=\prn{f^{\star}(1),\dots,f^{\star}(T)}$ is defined as:
$$
\apx_{\epsilon,T}(f^{\star})=\sum_t p_f^t\ell_f^t-(1+\epsilon)\sum_t \ell_{f^{\star}(t)}^t,
$$
and typical guarantees are $\En\brk{\apx(f^{\star})}=O\prn{ \nicefrac{K(f^{\star})\cdot\ln(dT)}{\epsilon}}$ where $K(f^{\star})=\sum_{t=2}^T \mathbbm{1}\crl{f^{\star}(t)\neq f^{\star}(t-1)}$ is the number of switches in the comparator. We show that any algorithm that can achieve such a guarantee even when $K(f^{\star})=2$ does not satisfy fairness in composition with respect to equalized error rate. This indicates that, for the fairness with equalized error rates purpose, the algorithm not being too good is essential.
\begin{theorem}\label{thm:shifting_algo_error_rates}
\vspace{0.3em}
For any $\alpha<\nicefrac{1}{2}$ and $\epsilon>0$, any algorithm that can achieve the vanishing approximate regret property against shifting comparators $f$ of length $K(f)=2$, running separate instances of the algorithm for each group is $\alpha$-unfair in composition with respect to equalized error rate. 
\end{theorem}
\begin{proof}
Our instance has two groups $\mathcal{G}=\crl{A,B}$, two experts $\mathcal{F}=\crl{f_1,f_2}$, and three phases. \begin{enumerate}
    \item Phase I lasts for half of the time horizon $\crl{1,\dots,\nicefrac{T}{2}}$ and during this time, we receive examples from group $A$. At round $t$, the adversary selects loss $\ell_{f}^t=1$ for the expert $f\in\mathcal{F}$ that is predicted with higher probability ($p_{f}^t\leq \nicefrac{1}{2}$) and $\ell_h^t=0$ for the other expert $h\in\mathcal{F}-\crl{f}$.  
    \item Phase II lasts $\sum_{\tau=1}^{\nicefrac{T}{2}}\ell_{f_1}^\tau$ rounds and the adversary selects losses $\ell_{f_1}^t=1$ and $\ell_{f_2}^t=0$. 
    \item Phase III lasts $\sum_{\tau=1}^{\nicefrac{T}{2}}\ell_{f_2}^\tau$ rounds and the adversary selects losses $\ell_{f_1}^t=0$ and $\ell_{f_2}^t=1$. 
\end{enumerate} 
Note that the instance is fair in isolation with respect to equalized error rates as the cardinality of both groups is the same (half of the population in each group) and the experts make the same number of mistakes in both groups.

By construction, the algorithm has expected average loss of at least $\nicefrac{1}{2}$ in members of group $A$.

We now focus on group $B$. By the shifting approximate regret guarantee and given that there exists a sequence of experts of length $2$ that has $0$ loss, it holds that the total loss of the algorithm needs to be sublinear on $T$ and, in particular, at most $(\nicefrac{1}{2}-\alpha)\cdot \frac{T}{2}$, which implies an expected error rate of $\nicefrac{1}{2}-\alpha$. Subtracting the two error rates concludes the proof.
\end{proof}

\section{Discussion}
\label{sec:discussion}
In this paper, we introduce the study of avoiding discrimination towards protected groups in online settings with non-i.i.d. examples. Our impossibility results for equalized odds consist of only two phases, which highlights the challenge in correcting for historical biases in online decision making.

Our work also opens up a quest towards definitions that are relevant and tractable in non-i.i.d. online settings for specific tasks. We introduce the notion of equalized error rates that can be a useful metric for non-discrimination in settings where all examples that contribute towards the performance also contribute towards fairness. This is the case in settings that all mistakes are similarly costly as is the case in speech recognition, recommender systems, or online advertising. However, we do not claim that its applicability is universal. For instance, consider college admission with two perfectly balanced groups that correspond to ethnicity (equal size of the two groups and equal number of positive and negatives within any group). A racist program organizer can select to admit all students of the one group and decline the students of the other, while satisfying equalized error rates -- this does not satisfy equalized odds. Given the impossibility result we established for equalized odds, it is interesting to identify definitions that work well for different tasks one encounters in online non-i.i.d. settings. Moreover, although our positive results extend to the case where predictors are vetted to be approximately non-discriminatory, they do not say anything about the case where the predictors do not satisfy this property. We therefore view our work as a first step towards better understanding non-discrimination in non-i.i.d. online settings.

\subsection*{Acknowledgements}
The authors would like to thank Manish Raghavan for useful discussions that improved the presentation of the paper. 

\bibliographystyle{alpha}
\bibliography{bib1}

\newcommand{\etalchar}[1]{$^{#1}$}
\begin{thebibliography}{FPCDG16}

\bibitem[ALMK16]{angwin2016machine}
Julia Angwin, Jeff Larson, Surya Mattu, and Lauren Kirchner.
\newblock Machine bias: There’s software used across the country to predict
  future criminals. {A}nd it’s biased against blacks.
\newblock {\em ProPublica}, 2016.

\bibitem[APJ16]{angwin2016facebook}
Julia Angwin and Terry Parris~Jr.
\newblock Facebook lets advertisers exclude users by race.
\newblock {\em ProPublica blog}, 28, 2016.

\bibitem[BBC{\etalchar{+}}16]{bird2016exploring}
Sarah Bird, Solon Barocas, Kate Crawford, Fernando Diaz, and Hanna Wallach.
\newblock Exploring or {E}xploiting? {S}ocial and {E}thical {I}mplications of
  {A}utonomous {E}xperimentation in {AI}.
\newblock In {\em Workshop on Fairness, Accountability, and Transparency in
  Machine Learning (FAT-ML)}, 2016.

\bibitem[BDNP18]{balcan_envyfree}
Maria-Florina Balcan, Travis Dick, Ritesh Noothigattu, and Ariel Procaccia.
\newblock Envy-free classification, 2018.

\bibitem[BG18]{buolamwini2018gender}
Joy Buolamwini and Timnit Gebru.
\newblock Gender shades: Intersectional accuracy disparities in commercial
  gender classification.
\newblock In {\em Conference on Fairness, Accountability and Transparency},
  2018.

\bibitem[BM05]{DBLP:conf/colt/BlumM05}
Avrim Blum and Yishay Mansour.
\newblock From external to internal regret.
\newblock In {\em Proceedings of the 18th Annual Conference on Learning Theory
  (COLT)}, 2005.

\bibitem[BS16]{barocas2014datas}
Solon Barocas and Andrew~D. Selbst.
\newblock {Big Data's Disparate Impact}.
\newblock {\em California Law Review}, 2016.

\bibitem[CDG18]{corbett2018measure}
Sam Corbett-Davies and Sharad Goel.
\newblock The measure and mismeasure of fairness: A critical review of fair
  machine learning.
\newblock {\em arXiv preprint arXiv:1808.00023}, 2018.

\bibitem[Cho17]{chouldechova2017fair}
Alexandra Chouldechova.
\newblock Fair prediction with disparate impact: A study of bias in recidivism
  prediction instruments.
\newblock {\em Big data}, 5(2):153--163, 2017.

\bibitem[CKP09]{Calders09buildingclassifiers}
Toon Calders, Faisal Kamiran, and Mykola Pechenizkiy.
\newblock Building classifiers with independency constraints.
\newblock In {\em IEEE International Conference on Data Mining (ICDM)}, 2009.

\bibitem[CV17]{DBLP:journals/corr/CelisV17}
L.~Elisa Celis and Nisheeth~K. Vishnoi.
\newblock Fair personalization.
\newblock In {\em Workshop on Fairness, Accountability, and Transparency in
  Machine Learning (FAT-ML)}, 2017.

\bibitem[DHP{\etalchar{+}}12]{DworkHPRZ2012awareness}
Cynthia Dwork, Moritz Hardt, Toniann Pitassi, Omer Reingold, and Richard Zemel.
\newblock Fairness through awareness.
\newblock In {\em Proceedings of the 3rd Innovations in Theoretical Computer
  Science Conference (ITCS)}, 2012.

\bibitem[DTD15]{datta2015automated}
Amit Datta, Michael~Carl Tschantz, and Anupam Datta.
\newblock Automated experiments on ad privacy settings.
\newblock {\em Proceedings on Privacy Enhancing Technologies}, 2015.

\bibitem[FFM{\etalchar{+}}15]{FeldmanFMSV2015disparate}
Michael Feldman, Sorelle~A. Friedler, John Moeller, Carlos Scheidegger, and
  Suresh Venkatasubramanian.
\newblock Certifying and removing disparate impact.
\newblock In {\em Proceedings of the 21th ACM SIGKDD International Conference
  on Knowledge Discovery and Data Mining (KDD)}, 2015.

\bibitem[FPCDG16]{feller2016computer}
Avi Feller, Emma Pierson, Sam Corbett-Davies, and Sharad Goel.
\newblock A computer program used for bail and sentencing decisions was labeled
  biased against blacks. it’s actually not that clear.
\newblock {\em The Washington Post}, 2016.

\bibitem[FS97]{Freund1997}
Yoav Freund and Robert~E Schapire.
\newblock A decision-theoretic generalization of on-line learning and an
  application to boosting.
\newblock {\em J. Comput. Syst. Sci.}, 1997.

\bibitem[GCGF16]{goh2016satisfying}
Gabriel Goh, Andrew Cotter, Maya Gupta, and Michael~P Friedlander.
\newblock Satisfying real-world goals with dataset constraints.
\newblock In {\em Advances in Neural Information Processing Systems (NIPS)},
  2016.

\bibitem[GJKR18]{gillen2018online}
Stephen Gillen, Christopher Jung, Michael Kearns, and Aaron Roth.
\newblock Online learning with an unknown fairness metric.
\newblock In {\em Advances in Neural Information Processing Systems (NIPS)},
  2018.

\bibitem[GM16]{gofer2016lower}
Eyal Gofer and Yishay Mansour.
\newblock Lower bounds on individual sequence regret.
\newblock {\em Machine Learning}, 103(1):1--26, 2016.

\bibitem[HPS16]{hardt2016equality}
Moritz Hardt, Eric Price, and Nati Srebro.
\newblock Equality of opportunity in supervised learning.
\newblock In {\em Advances in neural information processing systems (NIPS)},
  2016.

\bibitem[HW01]{DBLP:journals/jmlr/HerbsterW01}
Mark Herbster and Manfred~K. Warmuth.
\newblock Tracking the best linear predictor.
\newblock {\em Journal of Machine Learning Research}, 2001.

\bibitem[JKMR16]{JosephKMR2016fairness}
Matthew Joseph, Michael Kearns, Jamie~H Morgenstern, and Aaron Roth.
\newblock Fairness in learning: Classic and contextual bandits.
\newblock In {\em Advances in Neural Information Processing Systems (NIPS)},
  2016.

\bibitem[KCP{\etalchar{+}}17]{kilbertus2017avoiding}
Niki Kilbertus, Mateo~Rojas Carulla, Giambattista Parascandolo, Moritz Hardt,
  Dominik Janzing, and Bernhard Sch{\"o}lkopf.
\newblock Avoiding discrimination through causal reasoning.
\newblock In {\em Advances in Neural Information Processing Systems (NIPS)},
  2017.

\bibitem[KKM{\etalchar{+}}17]{Kannan:2017:FIM:3033274.3085154}
Sampath Kannan, Michael Kearns, Jamie Morgenstern, Mallesh Pai, Aaron Roth,
  Rakesh Vohra, and Zhiwei~Steven Wu.
\newblock Fairness incentives for myopic agents.
\newblock In {\em Proceedings of the 2017 ACM Conference on Economics and
  Computation (EC)}, 2017.

\bibitem[KLRS17]{KusnerLoRuSi2017}
Matt~J Kusner, Joshua Loftus, Chris Russell, and Ricardo Silva.
\newblock Counterfactual fairness.
\newblock In {\em Advances in Neural Information Processing Systems (NIPS)},
  2017.

\bibitem[KMM15]{kay2015unequal}
Matthew Kay, Cynthia Matuszek, and Sean~A Munson.
\newblock Unequal representation and gender stereotypes in image search results
  for occupations.
\newblock In {\em Proceedings of the 33rd Annual ACM Conference on Human
  Factors in Computing Systems}, 2015.

\bibitem[KMR17]{Kleinberg2017InherentTI}
Jon~M. Kleinberg, Sendhil Mullainathan, and Manish Raghavan.
\newblock Inherent trade-offs in the fair determination of risk scores.
\newblock In {\em Innovations of Theoretical Computer Science (ITCS)}, 2017.

\bibitem[KMR{\etalchar{+}}18]{KannanMoRoWaWu18}
Sampath Kannan, Jamie Morgenstern, Aaron Roth, Bo~Waggoner, and Zhiwei~Steven
  Wu.
\newblock A smoothed analysis of the greedy algorithm for the linear contextual
  bandit problem.
\newblock In {\em Advances in Neural Information Processing Systems (NIPS)},
  2018.

\bibitem[KNRW18]{kearns2017gerrymandering}
Michael Kearns, Seth Neel, Aaron Roth, and Zhiwei~Steven Wu.
\newblock Preventing fairness gerrymandering: Auditing and learning for
  subgroup fairness.
\newblock In {\em In Proceedings of the 35th International Conference on
  Machine Learning (ICML)}, 2018.

\bibitem[LDM16]{liu2016women}
Katherine~A Liu and Natalie~A Dipietro~Mager.
\newblock Women’s involvement in clinical trials: historical perspective and
  future implications.
\newblock {\em Pharmacy Practice}, 2016.

\bibitem[LDR{\etalchar{+}}18]{LiuDeRoSiHa18}
Lydia~T. Liu, Sarah Dean, Esther Rolf, Max Simchowitz, and Moritz Hardt.
\newblock Delayed impact of fair machine learning.
\newblock {\em 35th International Conference on Machine Learning (ICML)}, 2018.

\bibitem[LRD{\etalchar{+}}17]{DBLP:journals/corr/LiuRDMP17}
Yang Liu, Goran Radanovic, Christos Dimitrakakis, Debmalya Mandal, and David~C.
  Parkes.
\newblock Calibrated fairness in bandits.
\newblock {\em Workshop on Fairness, Accountability, and Transparency in
  Machine Learning (FAT-ML)}, 2017.

\bibitem[LS15]{DBLP:conf/colt/LuoS15}
Haipeng Luo and Robert~E. Schapire.
\newblock Achieving all with no parameters: Adanormalhedge.
\newblock In {\em Proceedings of The 28th Conference on Learning Theory
  (COLT)}, 2015.

\bibitem[LST16]{LykourisSyTa2016}
Thodoris Lykouris, Vasilis Syrgkanis, and \'{E}va Tardos.
\newblock Learning and efficiency in games with dynamic population.
\newblock In {\em Proceedings of the Twenty-seventh Annual ACM-SIAM Symposium
  on Discrete Algorithms (SODA)}, 2016.

\bibitem[LW94]{Littlestone:1994:WMA:184036.184040}
Nick Littlestone and Manfred~K. Warmuth.
\newblock The weighted majority algorithm.
\newblock {\em Inf. Comput.}, 108(2):212--261, February 1994.

\bibitem[PRT08]{Pedr08eshiRT20}
Dino Pedreshi, Salvatore Ruggieri, and Franco Turini.
\newblock Discrimination-aware data mining.
\newblock In {\em Proceedings of the 14th ACM SIGKDD International Conference
  on Knowledge Discovery and Data Mining (KDD)}, 2008.

\bibitem[RSWW18]{RaghavanSlWoWu18}
Manish Raghavan, Aleksandrs Slivkins, Jennifer~Vaughan Wortman, and
  Zhiwei~Steven Wu.
\newblock The externalities of exploration and how data diversity helps
  exploitation.
\newblock In {\em Proceedings of the 31st Conference On Learning Theory
  (COLT)}, 2018.

\bibitem[Swe13]{Sweeney2013}
Latanya Sweeney.
\newblock Discrimination in online ad delivery.
\newblock {\em Commun. ACM}, 56(5):44--54, May 2013.

\bibitem[WGOS17]{woodworth2017learning}
Blake Woodworth, Suriya Gunasekar, Mesrob~I Ohannessian, and Nathan Srebro.
\newblock Learning non-discriminatory predictors.
\newblock In {\em Conference on Learning Theory (COLT)}, 2017.

\bibitem[ZVRG17]{zafar2015learning}
Muhammad~Bilal Zafar, Isabel Valera, Manuel~Gomez Rodriguez, and Krishna~P
  Gummadi.
\newblock Learning fair classifiers.
\newblock {\em Proceedings of 30th Neural Information Processing Systems},
  2017.

\bibitem[ZWS{\etalchar{+}}13]{zemel2013learning}
Rich Zemel, Yu~Wu, Kevin Swersky, Toni Pitassi, and Cynthia Dwork.
\newblock Learning fair representations.
\newblock In {\em International Conference on Machine Learning (ICML)}, 2013.

\end{thebibliography}

\end{document}